%% file: main.tex
\RequirePackage{fix-cm}
\documentclass[12pt]{article}
\PassOptionsToPackage{table,dvipsnames}{xcolor}
\input{commands}

\usepackage[T1]{fontenc}
\usepackage[utf8]{inputenc}
\usepackage{amsmath,amssymb}
\usepackage{graphicx}
\usepackage{booktabs,array,tabularx,multirow,longtable}
\usepackage{threeparttable}
\usepackage{tabularx}
\usepackage{array}
\usepackage{tabularx}
\usepackage{array}
\newcolumntype{Y}{>{\centering\arraybackslash}X}
\usepackage{hhline}
\usepackage[table]{xcolor}
\usepackage{pifont}
\usepackage{algorithm,algorithmic}
\usepackage{wrapfig}
\usepackage{placeins}
\IfFileExists{microtype.sty}{\usepackage{microtype}}{}
\makeatletter
\@ifpackageloaded{natbib}{}{\usepackage[authoryear,round]{natbib}}
\makeatother
\usepackage{url}
\usepackage{hyperref}
\hypersetup{hidelinks}

\definecolor{gemrow}{RGB}{247,246,252}
\definecolor{groupgray}{RGB}{243,243,243}
\definecolor{groupbar}{RGB}{242,243,250}
\definecolor{gainGreen}{RGB}{0,120,70}
\definecolor{lossRed}{RGB}{190,45,45}
\newcommand{\gooddelta}[1]{\textcolor{gainGreen}{#1}}
\newcommand{\baddelta}[1]{\textcolor{lossRed}{#1}}

\newcommand{\gem}{\textsc{GEM}}
\title{Stable Geometry with Divergent Task Evidence for Efficient Long-Horizon Agent Compression}
\input{authors}
\date{}

\newcommand{\gemmethodcite}[1]{{\scriptsize\citep{#1}}}
\DeclareUnicodeCharacter{221E}{\ensuremath{\infty}}

\hypersetup{colorlinks=true,linkcolor=mydarkblue,citecolor=mydarkblue,urlcolor=mydarkblue}
\usepackage{team-template}
\renewcommand{\TeamPaperID}{STABLE GEOMETRY / LONG-HORIZON AGENTS}
\renewcommand{\TeamShortTitle}{Stable Geometry}
\begin{document}
\pagestyle{fancy}
\maketitle
\thispagestyle{first}

\begin{abstract}
Long horizon agents accumulate growing interaction histories that increase
context and inference costs. We find that geometric redundancy alone is an
insufficient criterion for safe compression. Although agent histories exhibit
strong low dimensional structure, similar global geometry can preserve very
different amounts of task evidence. At identical retained block counts,
evidence aware selection raises next action Top 3 retention from $0.31$ to
$0.69$, while centroid similarity remains $0.98$. Controlled replacement
further shows that action related information can be substantially altered
while global geometric measures remain nearly unchanged. Motivated by this
gap between geometry and evidence, we introduce
\textbf{Geometry Guided Evidence Preserving Memory (\gem{})}, a training free
compressor that protects task and execution evidence before using geometric
residuals to complete coverage. \gem{} reduces mean combined token usage from
$2.69$M to $2.11$M per task, a $21.4\%$ reduction, while maintaining comparable
task reward. Our results show that efficient agent history compression should
optimize for preserved task evidence rather than geometric coverage alone.
\end{abstract}

\section{Introduction}

Language-model agents continuously accumulate task instructions, plans, tool
calls, environment observations, errors, and intermediate artifacts throughout
long-horizon execution
\citep{yao2022react,liu2024agentbench,kang2025acon,hu2026sam}.
As the trajectory grows, repeatedly forwarding the full interaction history
increases context and inference costs
\citep{jiang2023llmlingua,jiang2024longllmlingua,kang2025acon}. Existing systems therefore summarize,
edit, retrieve, or selectively retain historical information
\citep{pan2024llmlingua,wang2026swe,li2026self,li2026sculptor}.
However, compression is not only a question of how much history can be removed.
A more fundamental question is which historical records can be discarded
without removing information that remains necessary for subsequent execution. This question is particularly difficult when redundant-looking records encode
different stages of task execution across prolonged sequences of interaction.

A natural approach is to preserve the dominant geometry of the history
representation. Prior studies show that contextual representations often
exhibit anisotropy and concentration along dominant directions
\citep{ethayarajh2019contextual,mu2017all,gao2021simcse}.
If most historical variation lies in a low-dimensional
subspace, a small subset of interaction blocks may appear sufficient to cover
the trajectory. Yet geometric redundancy does not imply execution
redundancy. A plan to modify a file, the command that performs the modification,
and a later partial read may occupy similar semantic directions while
establishing different facts about the task state. A historical block may
therefore contribute little new geometric variation while still carrying
information required for the next action. We refer to this mismatch as the gap between geometry and evidence.

This distinction makes history compression inherently evidence sensitive.
Recent agent-memory studies likewise emphasize that task state, execution
dependencies, and persistent evidence can remain important across long
interaction horizons
\citep{chen2026beyond,zhou2026mem1,wu2026remember,hu2026sam}.
Removing a geometrically redundant block is harmless only when its execution
role has also become redundant. In contrast, removing a low-novelty record
that contains a goal constraint, state transition, unresolved error, or
source-bearing observation may eliminate information that later actions still
depend on. Global representational coverage alone is therefore insufficient
for deciding which blocks are safe to remove. Compression should first
identify evidence that remains consequential for execution and only then use
geometric redundancy to decide which remaining records can be omitted. This suggests that compression should jointly account for representational
coverage and the execution role of retained records.

Motivated by these observations, we introduce
\textbf{Geometry-Guided Evidence-Preserving Memory (\gem{})}, a training-free
external compressor for long agent histories. \gem{} first protects task and
execution evidence, including recent context, goal-related observations,
execution-state records, fresh errors, and source-bearing information. It then
completes the retained set using geometric residuals so that uncovered
representation directions are preserved without allowing low geometric
novelty to override protected evidence. Retained interaction blocks remain
unchanged and are reconstructed in chronological order before being forwarded
to the acting model through its standard request interface.
The main contributions are as follows:

{
\setlength{\leftmargini}{1.2em}
\setlength{\itemsep}{3pt}

\begin{itemize}

    \item \textbf{Geometry and evidence gap in agent histories.}
    We distinguish global representational coverage from execution-relevant
    evidence retention and show that similar geometric coverage can correspond
    to substantially different preservation of next-action information.
    Controlled replacement further isolates low-energy evidence whose removal
    changes action-related information while leaving global geometry nearly
    unchanged.

    \item \textbf{Evidence-first geometric compression.}
We introduce a two-stage selection principle that protects consequential
task and execution evidence before applying geometric residual completion.
This prevents geometric redundancy from being treated as sufficient
evidence that a historical interaction is safely removable during long-horizon execution.

    \item \textbf{Execution-aware online history management.}
    We operate on complete interaction blocks, preserve call-result structure,
    and maintain original message content and chronological order. Online
    forwarding constraints limit each compression update while allowing newly
    activated evidence to remain available during subsequent execution.

\end{itemize}
}

\section{Related Work}
\label{sec:related}

\subsection{Context Management for Long Horizon Agents}

Methods for controlling growing agent histories differ in what they remove
and how they decide to remove it. ReSum periodically summarizes interaction
history, ACON improves compression guidelines using failure feedback, and PACE
adjusts historical detail according to predicted relevance to the next action
\citep{kang2025acon,wei2026pace}.
SelfCompact and Self-GC give agents context management operations, while
Sculptor supports summarizing, hiding, restoring, and searching historical
information \citep{li2026self,hao2026self,li2026sculptor}.
Other recent approaches learn what to retain, decouple context management from
the acting model, or expose agentic context-management mechanisms
\citep{jahan2026learning,zhang2026comem,li2026acm}. 
\gem{} instead selects complete original interaction blocks. It neither
produces a replacement summary nor requires the acting agent to learn new
memory operations.

\subsection{Memory and Execution Evidence}

Memory systems preserve information beyond the active context
\citep{packer2023memgpt,park2023generative,shinn2023reflexion}. MEM1 jointly
learns reasoning and memory consolidation, SAM combines compact memory cues
with access to raw history, and AdaCoM learns an external context manager for
a frozen agent \citep{zhou2026mem1,hu2026sam}.
MAGE organizes history around execution dependencies and errors, and Proactive
Memory Agent emphasizes that information present in a transcript may still
fail to influence later behavior \citep{chen2026beyond,wu2026remember}.
These studies motivate explicit treatment of task state. Our focus is more
specific: when original records share semantic content, their execution roles
may still differ. \gem{} protects such evidence before completing coverage
of the remaining history.

\subsection{Task-Aware Context Selection}

Prompt compression methods such as LLMLingua-2 learn which tokens to retain,
while SWE-Pruner and SWE-Pruner Pro study selective retention of tool context
\citep{pan2024llmlingua,wang2026swe}.
Dynamic token pruning and long-context evaluation further study which parts of
a long context remain useful under fixed model interfaces
\citep{fu2024lazyllm,hsieh2024ruler,zhang2024bench}. Retrieval systems likewise
explore hierarchical or adaptive evidence selection
\citep{sarthi2024raptor,jeong2024adaptive,asai2024self}.
Work on context interference further shows that retaining relevant evidence
is not sufficient when surrounding information impairs its use.
We study selection at the level of complete interaction blocks and distinguish
coverage of history representations from information related to the next
action. The selector is training-free and uses only the current task and available history
without modifying the acting model.

\section{Geometry and Task Information in Agent Histories}
\label{sec:geometry-task}

We analyze history representations to distinguish redundancy from information
needed for continued execution. The recorded next action is used only for
offline measurement and is never available to the deployed compressor.

\subsection{History Representation}
\label{sec:agent-history}

A removable interaction block contains an assistant tool-call message and all
of its corresponding tool results. Calls and results are retained or removed
together. The system and task prefix is kept separately, and messages that
cannot be safely assigned to complete blocks are not deleted.

Prior analyses of representation spaces characterize anisotropy, dominant
directions, and rank-based concentration in contextual and sentence embeddings
across modern neural representation models
\citep{ethayarajh2019contextual,mu2017all,gao2021simcse}.

Let $H_t=\{o_1,\ldots,o_n\}$ denote the available complete blocks. A frozen
encoder maps each block to a unit vector $x_i\in\mathbb R^d$. With
$X=[x_1,\ldots,x_n]^\top$ and mean representation $\bar x$, we study centered
variation through the following centered matrix decomposition
\begin{equation}
X_c=X-\mathbf 1\bar x^\top=U\Sigma V^\top,
\qquad \Sigma=\operatorname{diag}(\sigma_1,\ldots,\sigma_r).
\label{eq:history-svd}
\end{equation}

The singular values are ordered from largest to smallest. Writing
$p_j=\sigma_j^2/\sum_k\sigma_k^2$, we summarize concentration by
using two complementary spectral statistics
\begin{equation}
r_{\mathrm{eff}}(X_c)=\exp\Bigl(-\sum_j p_j\log p_j\Bigr),
\qquad
r_\rho(X_c)=\min\Bigl\{k:\sum_{j\le k}p_j\ge\rho\Bigr\}.
\label{eq:history-dimension}
\end{equation}

We also report $\bar r_\rho=r_\rho/n$. The spectral analysis sets $\rho$ to
$0.90$. Low effective dimension indicates concentration of variation, not
that every block outside a dominant representation can be safely removed
during subsequent downstream execution steps.

\subsection{Coverage and Action-Related Retention}
\label{sec:retained-information}

Let $S\subseteq[n]$ index a retained subset, $\bar x_S$ be its mean, and $Q_S$
contain an orthonormal basis for its span. Overall geometric coverage is
measured by
\begin{equation}
C(S)=\frac{\bar x_S^\top\bar x}{\|\bar x_S\|_2\|\bar x\|_2},
\qquad
E(S)=\frac{1}{n}\sum_{i=1}^n\|Q_S^\top x_i\|_2^2.
\label{eq:retained-geometry}
\end{equation}
Centroid similarity $C(S)$ measures preservation of the mean direction.
Captured energy $E(S)$ measures average projection onto the retained span.
Unlike the centered spectral analysis, $E(S)$ uses the original normalized
block vectors, matching the online selector.

For the unit representation $a$ of the recorded next action, we separately
measure
\begin{equation}
P_a(S)=\|Q_S^\top a\|_2^2,
\qquad
R_a(S)=\frac{P_a(S)}{P_a([n])},
\qquad
T_3(S;a)=\frac{|S\cap\mathcal N_3(a)|}{3}.
\label{eq:task-support}
\end{equation}
Here $\mathcal N_3(a)$ contains the three history blocks most similar to $a$.
Projection measures retention of the action direction; Top-3 retention
measures retention of the corresponding original blocks. Neither is a task
success rate. Ratios are evaluated only when their denominators are defined.

\subsection{Information Outside the Dominant Subspace}
\label{sec:low-energy}

Let $V_{\mathrm{high}}=[v_1,\ldots,v_{r_\rho}]$ and
$V_{\mathrm{low}}=[v_{r_\rho+1},\ldots,v_r]$. The latter directions account
for at most a fraction $1-\rho$ of centered history variation. Their share of
the action projection is
\begin{equation}
m_{\mathrm{low}}(a)=
\frac{\|V_{\mathrm{low}}^\top a\|_2^2}{\|V^\top a\|_2^2}.
\label{eq:low-energy-action}
\end{equation}
A large value indicates that action-related information extends beyond the
leading history directions. It does not identify a removable block by itself.
For the controlled replacement analysis, we additionally measure the
low-energy fraction of each centered block and its alignment with the action
in this subspace. Appendix~\ref{app:directional-intervention} gives the
complete definitions.

\section{Evidence Beyond the Dominant Subspace}
\label{sec:geometry-task-results}

We test whether agent histories contain non-random representational redundancy,
whether matching global geometry is sufficient to retain task evidence, and
whether action-related information can be selectively removed from low-energy
directions.

\subsection{Histories Exhibit Structured Redundancy}
\label{sec:lowdim-result}

The first analysis contains $171$ trajectories and $4{,}701$ complete blocks.
Each block is encoded independently using mean layer-32 representations from a
frozen Qwen3.5-9B encoder, producing $4{,}096$-dimensional vectors.  A second
analysis uses the $1{,}024$-dimensional Qwen3-Embedding-8B representations used by
the online selector on $32$ trajectories and $1{,}162$ block representations.
These are isolated block encodings or external selector embeddings, not hidden
states of the acting model under its full history.

For each trajectory, we compare the real representation matrix with two matched
controls. Independent circular feature shifts preserve every feature's marginal
values but disrupt alignment across blocks; a Gaussian control matches matrix
shape. Figure~\ref{fig:lowdim} gives the visual comparison, while the exact
centered effective-rank and $r_{90}/n$ medians are moved to
Appendix~\ref{app:geometry}, Table~\ref{tab:lowdim-main}. Median effective ranks
are $8.39$, $16.99$, and $19.96$ for real, feature-shift, and Gaussian isolated
blocks, and $14.04$, $23.54$, and $30.07$ in selector space. All $171$
isolated-block comparisons and all $32$ selector-space comparisons place the real
history below its feature-shift control in effective rank. Across the four
primary feature-shift comparisons, the reported Holm-adjusted values satisfy
$p\le1.86\times10^{-9}$.

\begin{figure}[t]
\centering
\includegraphics[width=0.98\linewidth]{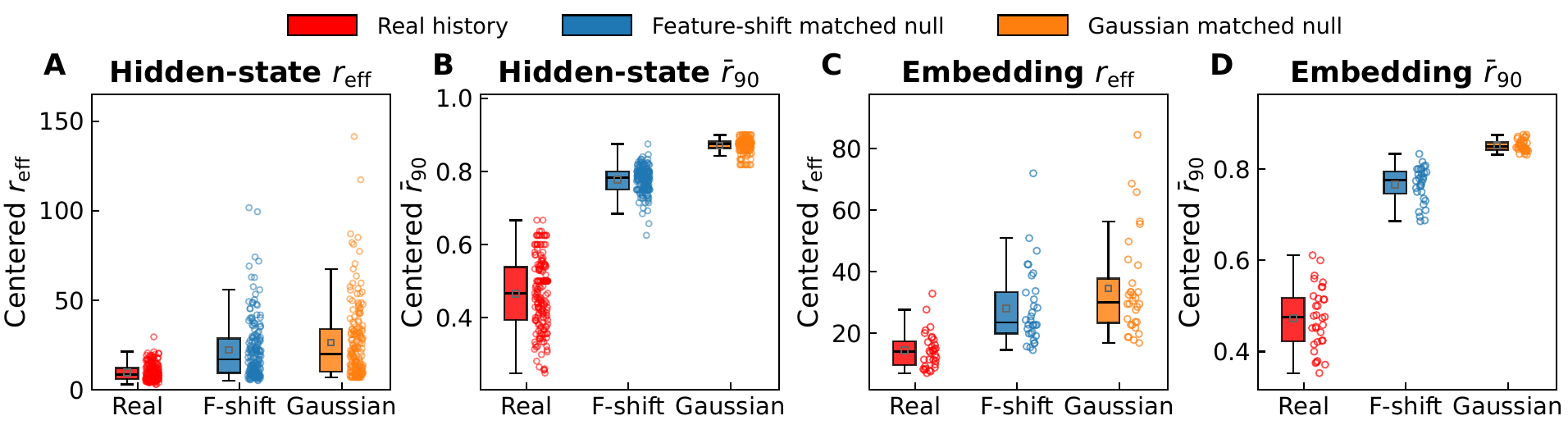}
\vspace{-10pt}
\caption{Structured redundancy in agent histories. Real histories are compared
with feature-shift and Gaussian controls. The two representation spaces
show the same concentration pattern.}
\label{fig:lowdim}
\end{figure}

The result establishes representational concentration, not safe removability.
Indeed, the uncentered greedy-coverage analysis uses a median of six isolated
blocks and ten selector blocks, corresponding to median retained fractions of
$0.30$ and $0.35$, respectively.  Those counts describe geometric coverage
only; they do not imply that the omitted blocks are task-irrelevant.

\subsection{Coverage Does Not Determine Evidence Retention}
\label{sec:matched-budget-result}

We next compare geometry-only and evidence-constrained selection at identical
retained block counts. The dense replay contains $650$ checkpoints from $32$
trajectories, starting after $16$ complete blocks. Individual trajectories
contribute between $2$ and $73$ checkpoints (median $15.5$), with $130$,
$108$, $161$, and $251$ checkpoints from Code, Office, Security, and Web,
respectively. The independent statistical units remain the $32$ trajectories.
At the same retained budget, evidence-constrained selection preserves
substantially more task information, particularly next-action Top-3 retention,
while global geometry remains nearly unchanged.
\Needspace{160pt}
\begin{wraptable}{r}{0.58\linewidth}
    \vspace{-2pt}
    \centering
    \caption{Matched-budget retention at 650 checkpoints.}
    \label{tab:matched-budget-main}
    \vspace{3pt}

    \footnotesize
    \setlength{\tabcolsep}{3.5pt}
    \renewcommand{\arraystretch}{1.08}

    \begin{tabularx}{\linewidth}{
    @{}
    >{\raggedright\arraybackslash}X
    |>{\centering\arraybackslash}p{0.20\linewidth}
    |>{\centering\arraybackslash}p{0.20\linewidth}
    |>{\centering\arraybackslash}p{0.20\linewidth}
    @{}
}
    \hline
    \textbf{Metric}
    & \textbf{Geometry}
    & \textbf{Evidence}
    & \textbf{$\Delta$} \\
    \hline

    Action projection
    & 0.91 & 0.94
    & \gooddelta{$\uparrow\,0.03$} \\

    Next-action Top-3
    & 0.31 & \textbf{0.69}
    & \gooddelta{$\uparrow\,0.38$} \\

    Final projection
    & 0.90 & 0.94
    & \gooddelta{$\uparrow\,0.04$} \\

    Goal projection
    & 0.91 & 0.94
    & \gooddelta{$\uparrow\,0.03$} \\
    \hline

    Centroid similarity
    & 0.98 & 0.98
    & \baddelta{$-$} \\
    \hline

    \end{tabularx}

    \vspace{0pt}
\end{wraptable}

Table~\ref{tab:matched-budget-main} reports the pooled same-budget means. At the
same retained count, next-action Top-3 retention rises from $0.31$ to $0.69$,
while centroid similarity is unchanged at two-decimal precision. Action, final-state, and
goal projections show the same pattern: evidence retention improves without a
comparably large shift in the global centroid measure. The metric definitions
and pooled-checkpoint aggregation are detailed in
Appendix~\ref{app:same-budget}.

The dense numbers are descriptive pooled checkpoint means, with trajectories
remaining the independent statistical units throughout the analysis. The small mean centroid change
alone does not establish statistical equivalence.

\begin{figure}[t]
\centering
\includegraphics[width=1.0\linewidth]{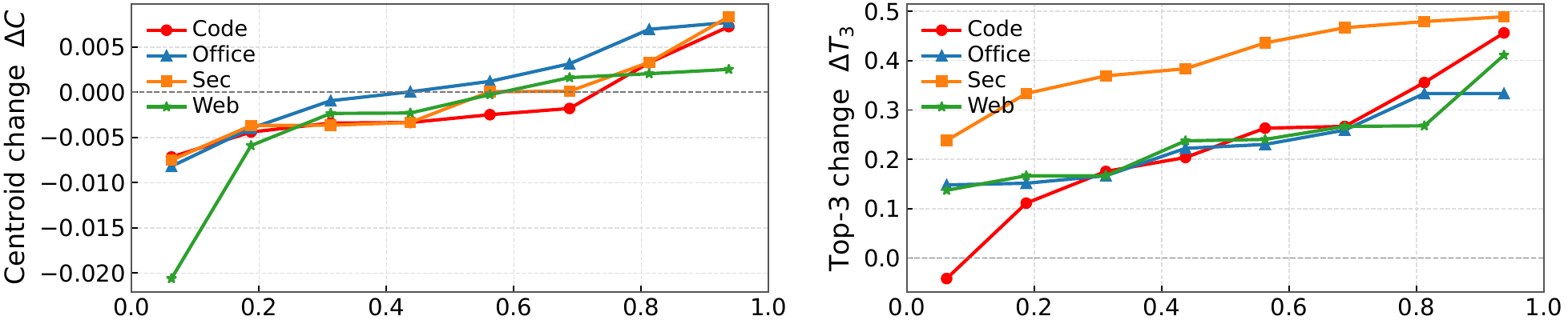}
\vspace{-20pt}
\caption{Matched-budget geometry and task information. The trajectory-level
visualization contrasts evidence-constrained and geometry-only selection at the
same retained block count.}
\label{fig:matched-dense}
\end{figure}

\subsection{Low-Energy Evidence and Controlled Replacement}
\label{sec:low-energy-result}

\begin{wrapfigure}{r}{0.49\linewidth}
    \vspace{-0.65\baselineskip}
    \centering
    \includegraphics[width=\linewidth]{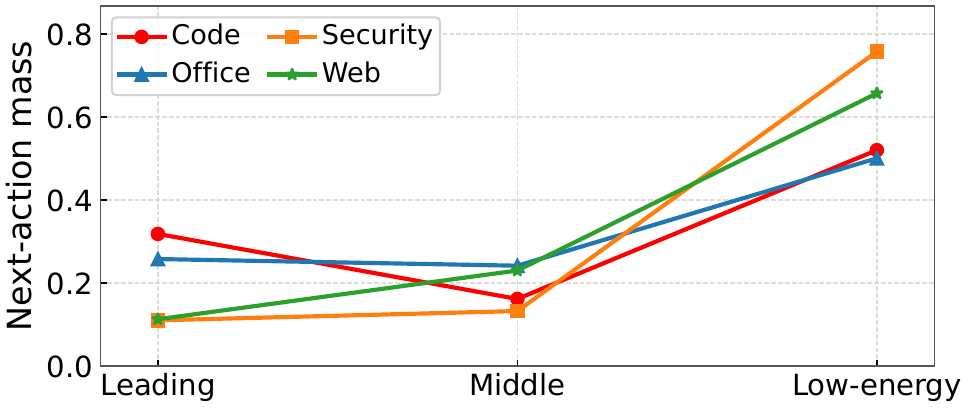}
    \vspace{-19pt}
    \caption{Next-action information across ordered spectral bands.}
    \label{fig:low-energy-spectrum}
    \vspace{0pt}
\end{wrapfigure}

Figure~\ref{fig:low-energy-spectrum} shows that next-action information is not
confined to the leading directions. The plotted bands are descriptive rank bands,
not equal-energy intervals; the formal intervention instead uses directions after
the $90\%$ centered-energy cutoff. A block can therefore be weakly represented by
the dominant variation of the trajectory while still carrying evidence relevant
to the next action.

This observation motivates a controlled single-block replacement. The targeted
condition removes a retained block whose low-energy component is strongly aligned
with the recorded next action. The matched condition removes a low-alignment block
while matching representation similarity, recency, and serialized length.

\Needspace{170pt}
\begin{wraptable}{r}{0.33\linewidth}
    \vspace{-0.45\baselineskip}
    \centering
    \vspace{-10pt}
    \caption{Targeted block replacement effects.}
    \label{tab:targeted-replacement}
\vspace{+5pt}
    \footnotesize
    \setlength{\tabcolsep}{3.0pt}
    \renewcommand{\arraystretch}{1.08}

    \begin{tabularx}{\linewidth}{
        @{}
        >{\raggedright\arraybackslash}X
        >{\centering\arraybackslash}p{0.23\linewidth}
        @{}
    }
    \hline

    \textbf{Metric} & $\boldsymbol{\Delta}$ \\
    \hline

    \rowcolor{groupgray}
    \multicolumn{2}{@{}l@{}}{\textit{Task information}} \\
    \hline

    Action projection   & \textbf{$-0.03$} \\
    Next-action Top-3   & \textbf{$-0.33$} \\
    \hline

    \rowcolor{groupgray}
    \multicolumn{2}{@{}l@{}}{\textit{Global geometry}} \\
    \hline

    Centroid similarity & $+0.00$ \\
    Captured energy     & $+0.00$ \\
    \hline

    \end{tabularx}

    \vspace{0pt}
\end{wraptable}

Table~\ref{tab:targeted-replacement} shows the same dissociation in intervention
form. Targeted replacement reduces action projection by $0.03$ and next-action
Top-3 retention by $0.33$, whereas centroid similarity changes by only $+0.00$
and captured energy by $+0.00$. The intervention quantities and matching
definition are given in Appendix~\ref{app:directional-intervention} for full reproducibility.

The targeted replacement therefore removes substantially more action-related
information without a comparable change in the global geometric measures. This is
a representation-level intervention rather than a rerun of the acting agent, so it
supports the distinction between geometry and evidence without by itself establishing a
causal task-reward effect.

Execution-aware protection distinguishes plans, state-changing operations, and
subsequent observations, even when they refer to the same artifact.
Appendix~\ref{app:evidence-protection} details the rules that identify and
preserve these different execution roles.
\section{\gem: Geometry-Guided Evidence-Preserving Memory}
\label{sec:method}

\gem{} is a training-free external compressor for complete interaction
histories. It separates records that should be protected from records selected
for additional geometric coverage. The encoder is frozen, and the acting
model is accessed through its ordinary request interface. No agent hidden
states, parameter updates, or generated summaries are required.

\begin{figure}[t]
\centering
\includegraphics[width=\linewidth]{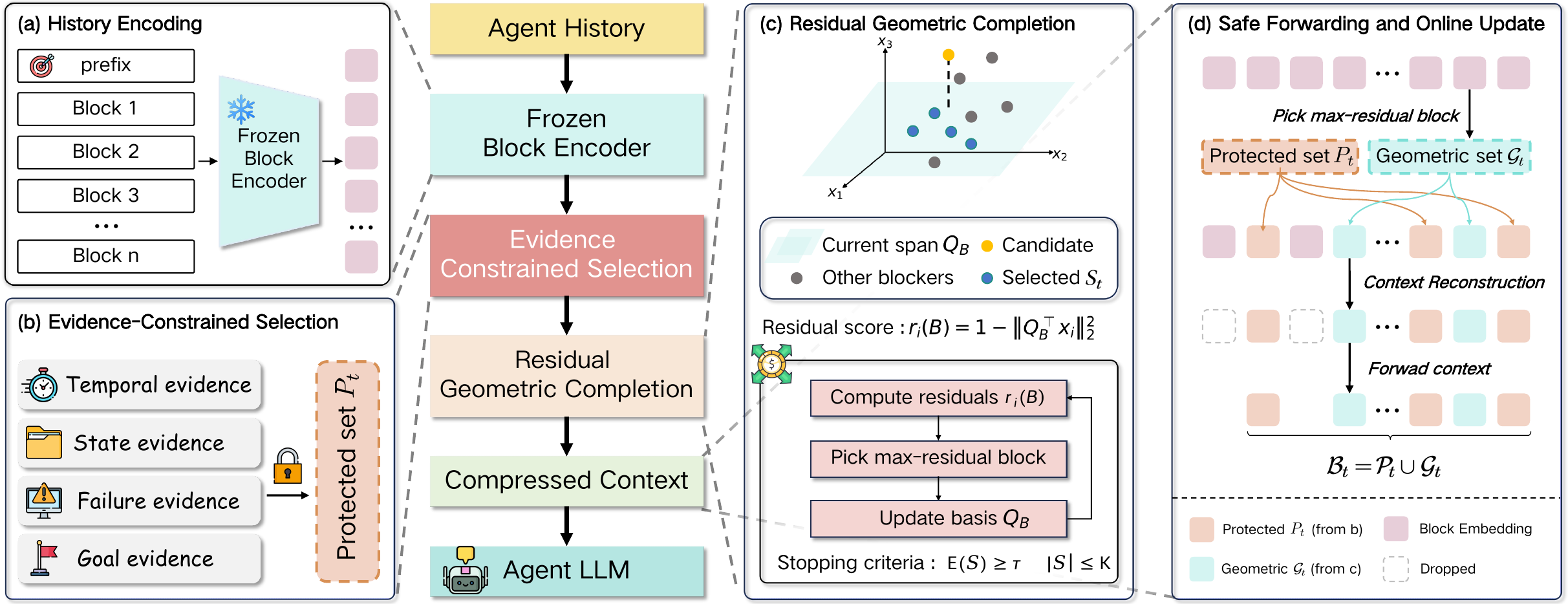}
\caption{Overview of \gem{}. (a) A frozen encoder represents complete
interaction blocks. (b) Evidence constraints form the protected set
$\mathcal P_t$. (c) Residual geometric completion adds blocks whose directions
are not covered by the current selected span. (d) Retained original blocks
are reconstructed in chronological order and forwarded to the unchanged
agent. The online forwarding constraints are described in
Section~\ref{sec:geometric-completion}.}
\label{fig:overview}
\end{figure}

Figure~\ref{fig:overview} summarizes the procedure. The protected set seeds
geometric completion; the two stages are not independent selectors.
Embeddings guide selection, while the original message records supply the
forwarded context. This separation keeps the selection mechanism lightweight while preserving
the original interaction semantics seen by the acting model.

\subsection{Evidence-Constrained Selection}
\label{sec:constrained-selection}

We distinguish the protected set $\mathcal P_t$, the selected core $B_t$, and
the final forwarded set $S_t$. The core contains the evidence and additional
directions identified by the selector. The forwarded set may be larger
because the online deletion guard retains additional original blocks. This distinction separates the semantic selection objective from the
request-level constraints that determine what is ultimately forwarded to the acting model.

For nominal core capacity $K$, let $K_t=\max\{K,|\mathcal P_t|\}$ so that
required evidence is not dropped when its count exceeds the nominal capacity.
The ideal core-selection problem is
\begin{equation}
B_t^\star=\operatorname*{arg\,min}_{B\subseteq[n]}|B|
\quad\mathrm{s.t.}\quad
\mathcal P_t\subseteq B,\quad E(B)\ge\tau,\quad |B|\le K_t.
\label{eq:constrained-compression}
\end{equation}
We use a greedy procedure rather than solve this problem exactly. It need
not find the smallest core, and the coverage target can remain unmet when
capacity or candidate directions are exhausted. Importantly, $K$ is not a
hard bound on the number of blocks forwarded to the agent.

\subsection{Protecting Task and Execution Evidence}
\label{sec:evidence-protection}

The protected set combines recent interactions, goal-related observations,
execution-state evidence, fresh errors, and source records
\citep{chen2026beyond,zhou2026mem1,wu2026remember,hu2026sam}. These groups are
combined and deduplicated before geometric completion. Recent blocks retain
the immediate working context. Goal selection keeps the eligible block most
similar to the current task. State selection uses the current state query,
file paths, and supported operation patterns to identify relevant records of
possible changes to task artifacts. Together, these signals prioritize records whose execution role may remain
important even when their semantic content is highly redundant.

A plan, a write, and a read of the same file can differ in execution value even
when they add little geometric novelty. We therefore inspect explicit editing
calls and supported Bash and Python write patterns, rather than relying only
on tool names. State candidates are organized by target path. Error evidence
is protected within a fixed age window, and source protection retains recent
records from distinct sources when source-like observations form a sufficient
fraction of the history.

These rules use only the available history and current goal. A detected write
is evidence of a possible mutation, not an independent verification that it
succeeded. Likewise, the expiration of an error record does not establish
that the error was resolved. The rules provide inexpensive protection cues,
not a complete symbolic model of the environment. Their recorded limits are
listed in Appendix~\ref{app:algorithm}.

\subsection{Geometric Completion and Online Updates}
\label{sec:geometric-completion}

Starting from $B=\mathcal P_t$, we construct an orthonormal basis $Q_B$ with
modified Gram-Schmidt orthogonalization. For each eligible unselected block,
we compute its residual and select the largest:
\begin{equation}
r_i(B)=1-\|Q_B^\top x_i\|_2^2,
\qquad i^\star=\operatorname*{arg\,max}_{i\notin B}r_i(B).
\label{eq:residual-selection}
\end{equation}
We add $i^\star$, update the basis, and continue until coverage, capacity, or
candidate exhaustion terminates selection. The residual therefore measures information not yet represented by the
current geometric span, rather than relevance to the task itself. Writing
$\mathcal G_t=B_t\setminus\mathcal P_t$ for the blocks added by geometric
completion gives $B_t=\mathcal P_t\cup\mathcal G_t$. A protected record cannot
be removed merely because it has a small residual. The default coverage target
is $\tau=0.90$, and the nominal core capacity is $K=16$.

The online implementation limits how much of a request can be removed in one
rewrite. If $L(S)$ is the serialized length of the complete request containing
blocks $S$, the forwarded set satisfies
\begin{equation}
B_t\subseteq S_t\subseteq[n],
\qquad L(S_t)\ge(1-\delta)L([n]).
\label{eq:forwarded-guard}
\end{equation}
We set $\delta=0.05$. Starting from the full safe request, the implementation
attempts to remove eligible blocks outside the core, prioritizing larger
redundant blocks while respecting this guard. Complete blocks are indivisible,
so an event may remove nothing. All retained messages remain unchanged and in
chronological order.

Requests with fewer than $16$ complete blocks are left unchanged. Global
selection occurs at the first eligible request, after a session or goal change,
or after $256$ newly added blocks. Between these events, the implementation
keeps the previously forwarded set and appends new blocks and newly activated
evidence. This incremental policy avoids repeatedly reconstructing the retained history
when no global reselection is required. Storing the actual forwarded set prevents a later update from
mistakenly removing blocks retained by the deletion guard. The per-request
character limit is distinct from cumulative token savings over a task.
Appendix~\ref{app:theory} gives the corresponding geometric derivations and
forwarding properties.

\section{Experiments}
\label{sec:experiments}

\subsection{Experimental Setup}
\label{sec:experimental-setup}
\begin{table}[h]
\centering
\caption{WorkBuddyBench Full260 results.
$S$ denotes mean reward $\times100$, and $T$ denotes combined tokens
per task (M). Higher $S$ and lower $T$ are better.}
\label{tab:main260}
\vspace{5pt}
\footnotesize
\setlength{\tabcolsep}{2.6pt}
\renewcommand{\arraystretch}{1.10}

\begin{tabularx}{\linewidth}{
    @{}
    >{\raggedright\arraybackslash}p{0.145\linewidth}
    | YY
    | YY
    | YY
    | YY
    | YY
    @{}
}

\hline

\multirow{2}{*}{\textbf{Method}}
& \multicolumn{2}{c|}{\textbf{Code}}
& \multicolumn{2}{c|}{\textbf{Office}}
& \multicolumn{2}{c|}{\textbf{Security}}
& \multicolumn{2}{c|}{\textbf{Web}}
& \multicolumn{2}{c}{\textbf{Overall}}
\\

\cline{2-11}

&
$S\uparrow$ & $T\downarrow$
&
$S\uparrow$ & $T\downarrow$
&
$S\uparrow$ & $T\downarrow$
&
$S\uparrow$ & $T\downarrow$
&
$S\uparrow$ & $T\downarrow$
\\

\hline

Base agent
& \textbf{76.99} & 1.37
& 81.84 & 1.20
& 47.76 & 5.98
& \textbf{72.14} & 2.43
& 69.87 & 2.69
\\

\rowcolor{gemrow}
\textbf{\gem{}}
& 75.25 & \textbf{1.35}
& \textbf{84.06} & \textbf{1.15}
& \textbf{49.56} & \textbf{3.95}
& \textbf{72.14} & \textbf{2.09}
& \textbf{70.18} & \textbf{2.11}
\\

\hline

\end{tabularx}

\vspace{2pt}
{\scriptsize
}

\end{table}

WorkBuddyBench Full260~\citep{team2026tencent} contains $260$ tasks across Code, Office, Security, and
Web. The main online evaluation uses DeepSeek-V4-Flash~\citep{xu2026deepseek} at temperature zero,
while \gem{} uses frozen $1{,}024$-dimensional Qwen3-Embedding-8B
representations. Task quality is reported as
\emph{Score}$=100\times$mean reward, and combined tokens include model input,
model output, and embedding tokens. Full260 is the primary end-to-end
evaluation, while the fixed-40 panel serves as a separate reference for common
compression controls. The same-ID \gem{} slice from the Full260 carrier is
marked separately in Table~\ref{tab:method-comparison}.
Appendix~\ref{app:fixed40-reference} describes the reference-panel protocol.

\subsection{Full260 End-to-End Results}
\label{sec:main-results}

Across all $260$ tasks, mean combined usage decreases from $2.69$M to $2.11$M
per task, a reduction of approximately $21.4\%$.  Mean score changes from
$69.87$ to $70.18$ (reward $\times100$).  The average hides
important domain variation.  Security shows the largest token decrease,
$5.98$M to $3.95$M per task, while its mean reward changes from $0.48$ to
$0.50$.  Office also increases in mean reward while using fewer tokens, and
Web keeps the same rounded reward while reducing tokens.  Code is the clearest
boundary case: tokens change only slightly ($1.37$M to $1.35$M) and mean reward
decreases by about $0.02$.  This pattern is consistent with code tasks being
more sensitive to intermediate execution and file-state evidence, but the domain
means alone do not identify the causal mechanism.

The token decomposition in Appendix~\ref{app:token-accounting} shows that the
net reduction comes primarily from a lower cached-input volume.  \gem{} also
introduces embedding usage and increases uncached input, so its savings are not
free; the cached-input reduction must exceed these added costs.
\begin{table}[t]
\centering
\caption{Fixed-40 context-management reference.
$S$ is mean reward $\times100$ and $T$ is combined tokens per task (M).
Each domain contains ten tasks.}
\label{tab:method-comparison}

\footnotesize
\setlength{\tabcolsep}{2.4pt}
\renewcommand{\arraystretch}{1.08}

\begin{tabularx}{\linewidth}{
@{}l|YY|YY|YY|YY|YY@{}
}
\hline

\multirow{2}{*}{\textbf{Method}}
& \multicolumn{2}{c|}{\textbf{Code}}
& \multicolumn{2}{c|}{\textbf{Office}}
& \multicolumn{2}{c|}{\textbf{Security}}
& \multicolumn{2}{c|}{\textbf{Web}}
& \multicolumn{2}{c}{\textbf{Overall}}
\\

\cline{2-11}

&
$S\uparrow$ & $T\downarrow$
&
$S\uparrow$ & $T\downarrow$
&
$S\uparrow$ & $T\downarrow$
&
$S\uparrow$ & $T\downarrow$
&
$S\uparrow$ & $T\downarrow$
\\

\hline

Base agent
& 72.93 & 1.44
& \underline{81.64} & 1.04
& 44.58 & 3.97
& 71.00 & 1.10
& 67.54 & 1.89
\\

\hline

Window ($K=5$)
& 66.81 & 3.12
& 67.43 & 3.98
& 32.32 & 2.13
& 59.00 & 2.58
& 56.39 & 2.95
\\

Window ($K=10$)
& 74.37 & 1.76
& 73.05 & 1.82
& 40.07 & \underline{0.82}
& 61.00 & 3.55
& 62.12 & 1.99
\\

Window ($K=20$)
& \underline{79.31} & 1.42
& \textbf{85.11} & 2.23
& 40.28 & 1.13
& 70.00 & 3.30
& 68.67 & 2.02
\\

Periodic summary ($n=3$)
& 65.44 & 1.91
& 67.66 & 2.14
& 26.32 & 0.99
& \textbf{82.66} & 3.82
& 60.52 & 2.21
\\

Periodic summary ($n=5$)
& \textbf{83.60} & 1.76
& 75.74 & 2.55
& 40.03 & 1.29
& 81.32 & 3.88
& \underline{70.17} & 2.37
\\

\hline

PACE~\gemmethodcite{wei2026pace}
& 70.38 & \underline{0.94}
& 55.75 & 1.01
& 43.53 & \textbf{0.57}
& 64.02 & 1.89
& 58.42 & \underline{1.10}
\\

LLMLingua-2~\gemmethodcite{pan2024llmlingua}
& 64.64 & 1.63
& 73.31 & 3.41
& 40.10 & 3.49
& 70.00 & 1.09
& 62.01 & 2.41
\\

SelfCompact~\gemmethodcite{li2026self}
& 76.83 & 1.57
& 79.62 & \underline{0.85}
& 36.19 & 3.14
& 71.00 & 0.71
& 65.91 & 1.57
\\

ACON-Core~\gemmethodcite{kang2025acon}
& 71.79 & 1.48
& 68.55 & 2.15
& \underline{47.94} & 1.51
& 65.00 & \underline{0.47}
& 63.32 & 1.40
\\

Self-GC~\gemmethodcite{hao2026self}
& 63.45 & 1.17
& 71.67 & 1.18
& 43.70 & 3.44
& 71.00 & 0.84
& 62.46 & 1.66
\\

LRE~\gemmethodcite{jahan2026learning}
& 62.62 & 2.26
& 63.44 & 3.47
& 27.14 & 2.85
& 68.00 & 0.59
& 55.30 & 2.29
\\

CoMem~\gemmethodcite{zhang2026comem}
& 56.43 & \textbf{0.55}
& 59.86 & \textbf{0.75}
& 2.50 & 1.00
& 58.00 & 0.76
& 44.20 & \textbf{0.77}
\\

SAM~\gemmethodcite{hu2026sam}
& 67.62 & 1.41
& 81.63 & 1.09
& 47.23 & 1.91
& 67.00 & 0.97
& 65.87 & 1.35
\\

SWE-Pruner~\gemmethodcite{wang2026swe}
& 63.45 & 2.17
& 72.60 & 1.43
& 40.87 & 1.82
& 65.00 & 0.58
& 60.48 & 1.50
\\

Sculptor~\gemmethodcite{li2026sculptor}
& 51.19 & 1.03
& 81.62 & 0.99
& 35.24 & 2.45
& 69.00 & 0.71
& 59.26 & 1.29
\\

ACM~\gemmethodcite{li2026acm}
& 38.69 & 2.54
& 55.93 & 1.53
& 6.19 & 5.20
& 34.00 & \textbf{0.19}
& 33.70 & 2.36
\\

\hline
\rowcolor{gemrow}
\textbf{\gem{}$^\dagger$}
& 76.67 & 1.74
& 78.51 & 1.32
& \textbf{55.31} & 3.88
& \underline{82.00} & 2.08
& \textbf{73.12} & 2.26
\\

\hline
\end{tabularx}

\vspace{2pt}
\end{table}

\FloatBarrier
\subsection{Fixed-40 Context-Management Reference Panel}
\label{sec:method-comparison}

Table~\ref{tab:method-comparison} reports reward and token use for the fixed-40
reference panel in each domain. The panel is useful for understanding how
tradeoffs between reward and token use vary across Code, Office, Security, and Web, but it is
\emph{not} used as a matched \gem{} comparison because the available \gem{}
same-ID slice belongs to the Full260 carrier snapshot. The fixed-40 evaluation
protocol is described in Appendix~\ref{app:fixed40-reference}.

The reference panel illustrates why compression should not be summarized by a
single token ratio.  For example, periodic summarization at $n=5$ has a relatively
high mean score in this panel but uses more tokens than the uncompressed base,
whereas CoMem uses substantially fewer tokens with a much lower mean score.
The domain columns show that these operating points vary substantially across
task families.

\FloatBarrier
\subsection{Online Ablation and Compression Strength}
\label{sec:ablation}

The online ablation uses $20$ formal-valid tasks, five from each domain.  Random
completion retains the same evidence rules and candidate pool but replaces
largest-residual completion with random selection.  Geometry-only removes the
evidence-protection bundle while retaining residual selection and the online
bookkeeping required for valid requests.

\begin{table}[t]
\centering
\caption{Online 20-task ablation by domain.
$S$ is mean reward $\times100$ and $T$ is combined tokens per task (M).
Each domain contains five tasks.}
\label{tab:online-ablation}
\vspace{5pt}
\footnotesize
\setlength{\tabcolsep}{2.4pt}
\renewcommand{\arraystretch}{1.08}

\begin{tabularx}{\linewidth}{
@{}l|YY|YY|YY|YY|YY@{}
}
\hline

\multirow{2}{*}{\textbf{Variant}}
& \multicolumn{2}{c|}{\textbf{Code}}
& \multicolumn{2}{c|}{\textbf{Office}}
& \multicolumn{2}{c|}{\textbf{Security}}
& \multicolumn{2}{c|}{\textbf{Web}}
& \multicolumn{2}{c}{\textbf{Overall}}
\\

\cline{2-11}

&
$S\uparrow$ & $T\downarrow$
&
$S\uparrow$ & $T\downarrow$
&
$S\uparrow$ & $T\downarrow$
&
$S\uparrow$ & $T\downarrow$
&
$S\uparrow$ & $T\downarrow$
\\

\hline

Base agent
& 59.79 & \textbf{2.33}
& \textbf{74.28} & \underline{3.02}
& \textbf{62.95} & 10.70
& 46.00 & 4.73
& \textbf{60.75} & 5.19
\\

\hline

\rowcolor{gemrow}
\textbf{\gem{}}
& \textbf{70.79} & 2.52
& \underline{69.32} & 3.18
& 53.45 & \textbf{4.33}
& \underline{48.00} & \underline{3.41}
& \underline{60.39} & \textbf{3.36}
\\

\hline

Random completion
& 47.40 & 2.72
& 60.86 & 5.20
& \underline{56.69} & \underline{6.69}
& 47.30 & \textbf{3.23}
& 53.06 & 4.46
\\

Geometry only
& \underline{60.93} & \underline{2.40}
& 59.62 & \textbf{1.25}
& 48.92 & 8.14
& \textbf{50.64} & 4.71
& 55.03 & \underline{4.12}
\\

\hline
\end{tabularx}

\vspace{2pt}
\end{table}

On this batch, the full method scores $60.39$ versus $60.75$ for the base agent
(a $0.36$-point difference) while using about $35.30\%$ fewer combined tokens overall. The domain columns
make the tradeoff explicit: on the five Code tasks \gem{} raises reward while
slightly increasing tokens, whereas on the five Security tasks it sharply reduces
tokens but loses reward. Relative to random completion and geometry-only selection,
the full method has the highest overall reward and the lowest overall token use in
this batch. The ablation therefore supports the combined selector without implying
a universally monotone relationship between reward and compression.

\begin{figure}[t]
\centering
\includegraphics[width=\linewidth]{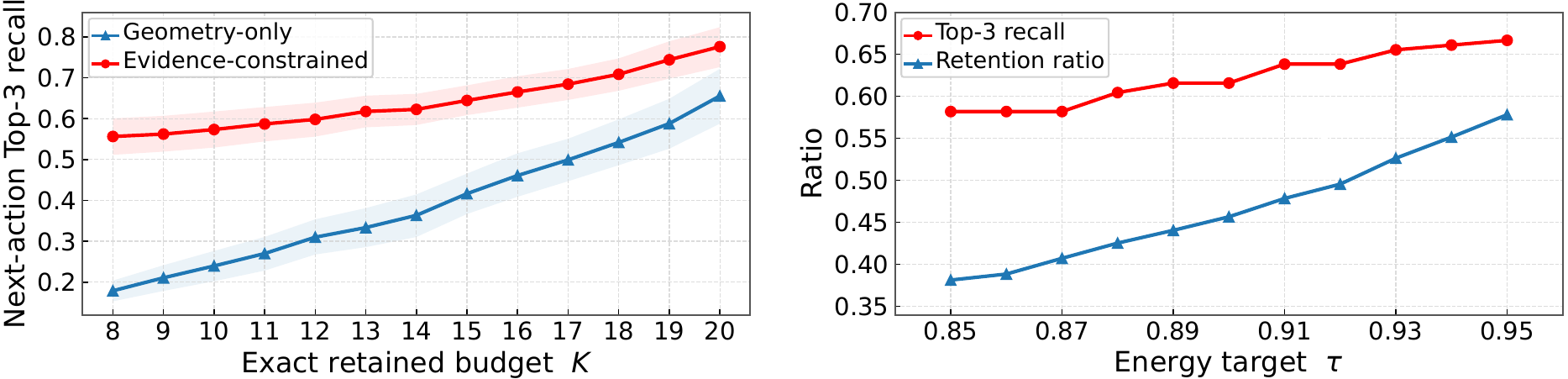}
\vspace{-20pt}
\caption{Compression strength.  (a) Next-action Top-3 retention at exact retained
budgets $K$.  (b) Top-3 retention and retained-history ratio as the coverage
threshold $\tau$ increases.  The exact-$K$ sweep uses $28$ trajectories and
$525$ common checkpoints; the coverage sweep uses $32$ trajectories and $59$
checkpoints.}
\label{fig:sensitivity}
\end{figure}

Evidence-constrained selection has higher Top-3 retention than geometry-only
selection at every tested exact budget $K\in[8,20]$.  Increasing $\tau$ from
$0.85$ to $0.95$ raises the retained fraction from $0.38$ to $0.53$ and
Top-3 retention from $0.58$ to $0.65$.  The default $\tau=0.90$ is therefore
an operating point on an observed tradeoff between retention and evidence, not a claim of a
universally optimal threshold.  The complete sweep values are listed in
Appendix~\ref{app:sweeps}.

\section{Conclusion}
\label{sec:conclusion}

Agent histories can be highly redundant in representation space while still
containing low-energy records that matter for continued execution.  This paper
separates those two notions of importance.  Same-budget selection and controlled
replacement show that preserving dominant geometry does not guarantee preservation
of next-action evidence.  \gem{} converts this observation into a simple external
compressor: protect task and execution evidence first, then use geometric
residuals to complete coverage.  It preserves original messages and leaves the
acting model unchanged.  On WorkBuddyBench Full260, the reported operating point
reduces mean combined tokens from $2.69$M to $2.11$M per task while maintaining
a similar aggregate reward, with clear domain-dependent tradeoffs.  The broader
lesson is that repeated content can still be consequential evidence; effective
history compression should model that distinction explicitly.

\subsection*{Reproducibility Statement}
We provide the methodological definitions, selection algorithm, implementation settings, evaluation protocols, ablation configurations, and additional analyses required to reproduce the reported results in the main text and appendix. The acting model and representation encoder remain frozen under the default setting, and all reported comparisons follow the evaluation and token-accounting procedures described in the manuscript.

\subsection*{AI Use Statement}
Generative AI tools assisted with code development, analysis scripts, figure and table preparation, and manuscript editing under author supervision. All experimental results and scientific claims were reviewed and verified by the authors, who take responsibility for the final manuscript.

\FloatBarrier
\clearpage
\setlength{\bibsep}{.5ex plus .8ex}
\IfFileExists{iclr2027_conference.bib}{%
  \IfFileExists{iclr2027_conference.bst}{%
    \bibliographystyle{unsrtnat}
  }{\bibliographystyle{unsrtnat}}
  \bibliography{iclr2027_conference}
}{%
  \section*{References}
  The bibliography is supplied by the existing
  \texttt{iclr2027\_conference.bib} file.  Citation keys are retained from the
  source project so this file can drop directly into the existing manuscript.
}

\clearpage
\clearpage
\appendix
\numberwithin{equation}{section}
\numberwithin{table}{section}
\newtheorem{gemprop}{Proposition}[section]
\newenvironment{gemproof}{\par\noindent\textit{Proof.} }{\hfill$\square$\par\medskip}


\section{Mathematical Derivations and Properties}
\label{app:theory}

We derive the geometric and structural properties used by \gem{}.
Throughout this section, the available history is fixed, the block vectors
$x_i\in\mathbb R^d$ have unit norm, and $Q_B$ contains an orthonormal basis
for the span of the selected vectors. The statements concern these
representation-space quantities and the selection and forwarding operations.

\subsection{Residual Selection and Reconstruction Error}
\label{app:projection-derivation}

Write $X=[x_1,\ldots,x_n]^\top$ and $P_B=Q_BQ_B^\top$, with
$P_\varnothing=0$. The selected block identities and the numerical basis are
maintained separately: a protected block remains selected even if its vector
is already in the span of the other selected vectors.
For any block, the orthogonal decomposition is
\begin{equation}
 x_i=P_Bx_i+e_i(B),\qquad e_i(B)=(I-P_B)x_i.
 \label{eq:app-residual-vector}
\end{equation}

\begin{gemprop}[Residual score and geometric coverage]
For normalized block vectors,
\begin{align}
 r_i(B)&=\|e_i(B)\|_2^2
       =1-\|Q_B^\top x_i\|_2^2,\label{eq:app-residual-identity}\\
 E(B)&=\frac{\|XQ_B\|_F^2}{n}
      =1-\frac{\|X(I-P_B)\|_F^2}{n}.
 \label{eq:app-coverage-error}
\end{align}
Consequently, $0\le r_i(B)\le1$, $0\le E(B)\le1$, and a core satisfying
$E(B)\ge\tau$ has mean squared embedding reconstruction error at most
$1-\tau$.
\end{gemprop}
\begin{gemproof}
Because $P_B$ is a symmetric idempotent projector, $P_Bx_i$ is orthogonal
to $(I-P_B)x_i$. Thus
$\|x_i\|_2^2=\|P_Bx_i\|_2^2+\|e_i(B)\|_2^2$.
Orthonormality gives $\|P_Bx_i\|_2^2=\|Q_B^\top x_i\|_2^2$.
Using $\|x_i\|_2^2=1$ proves the first identity. Summing it over all
$n$ blocks gives the second identity and the stated bounds.
\end{gemproof}

Equation~\ref{eq:app-residual-identity} explains the completion rule: the
largest-residual candidate contributes a direction that is poorly represented
by the current selected span. Equation~\ref{eq:app-coverage-error} gives the
coverage target its direct reconstruction interpretation. These quantities
are evaluated on the original normalized vectors, matching the online
selector; the centered spectral analysis is defined separately in
Appendix~\ref{app:geometry}.

\subsection{Coverage Gain and Incremental Basis Updates}
\label{app:coverage-gain}

For an eligible candidate $j$ with $r_j(B)>0$, modified Gram Schmidt
constructs the new unit direction
\begin{equation}
 q=\frac{x_j-Q_B(Q_B^\top x_j)}{\sqrt{r_j(B)}},
 \qquad Q_{B\cup\{j\}}=[Q_B,q].
 \label{eq:app-new-direction}
\end{equation}

\begin{gemprop}[Monotone coverage completion]
For the update in Eq.~\ref{eq:app-new-direction},
\begin{align}
 P_{B\cup\{j\}}&=P_B+qq^\top,\label{eq:app-projector-update}\\
 E(B\cup\{j\})-E(B)
 &=\frac1n\sum_{i=1}^n(q^\top x_i)^2,\label{eq:app-energy-gain}\\
 r_i(B\cup\{j\})&=r_i(B)-(q^\top x_i)^2.
 \label{eq:app-residual-update}
\end{align}
In particular, the coverage gain is at least $r_j(B)/n>0$.
More generally, $B\subseteq S$ implies $E(S)\ge E(B)$.
\end{gemprop}
\begin{gemproof}
The normalized residual $q$ is orthogonal to every column of $Q_B$.
The projector onto the expanded span is therefore $P_B+qq^\top$.
Substituting this expression into the definition of $E$ gives
Eq.~\ref{eq:app-energy-gain}, and subtracting the new projected component
gives Eq.~\ref{eq:app-residual-update}. All terms in the coverage gain
are nonnegative, and the $i=j$ term equals $r_j(B)$.
For nested sets $B\subseteq S$, extend $Q_B$ by an orthonormal basis
of the additional directions in $\operatorname{span}(S)$ and apply the
same argument to each added direction.
\end{gemproof}

For $m$ added independent directions $q_1,\ldots,q_m$, the increments
sum to
\begin{equation}
 E(B_m)=E(B_0)+\frac1n\sum_{k=1}^{m}\|Xq_k\|_2^2.
 \label{eq:app-cumulative-coverage}
\end{equation}
Thus geometric completion progressively covers additional directions after
the evidence set has been protected. The residual-update identity also
allows candidate scores to be updated without rebuilding all previous
projections at each iteration.

\subsection{Evidence Preservation and Forwarding Invariants}
\label{app:invariants}

The selector begins with $B_0=\mathcal P_t$ and uses the expanded capacity
$K_t=\max\{K,|\mathcal P_t|\}$. It only adds indices during completion.
The forwarding stage starts from the full valid request and considers
removing indices outside the completed core $B_t$.

\begin{gemprop}[Protected and structurally valid forwarding]
For an accepted global rewrite,
\begin{equation}
 \mathcal P_t\subseteq B_t\subseteq S_t\subseteq[n],\qquad
 L(S_t)\ge(1-\delta)L([n]).
 \label{eq:app-forwarding-invariants}
\end{equation}
Every retained removable block contains its complete assistant call and result group, and the retained messages preserve their original chronological order.
Between global reselections, under stable block identities, the append rule
satisfies
\begin{equation}
 S_{t+1}=S_t\cup\mathcal N_{t+1}\cup\mathcal A_{t+1},
 \qquad S_t\subseteq S_{t+1},
 \label{eq:app-append-invariant}
\end{equation}
where $\mathcal N_{t+1}$ and $\mathcal A_{t+1}$ are newly available blocks
and newly activated evidence, respectively.
\end{gemprop}
\begin{gemproof}
The core is initialized to $\mathcal P_t$ and is modified only by insertion,
so $\mathcal P_t\subseteq B_t$. Forwarding starts from $[n]$ and deletes
only members of its complement to $B_t$, proving $B_t\subseteq S_t$.
Initially the length constraint holds; accepting each proposed deletion
only when it continues to hold proves the bound by induction. Parsing
constructs complete call and result blocks, selection changes only whole-block
membership, and reconstruction sorts retained messages by their original
indices. This preserves call and result groups and chronological order.
Finally, the ordinary update is a set union with the saved actual forwarded
set, proving Eq.~\ref{eq:app-append-invariant}.
\end{gemproof}

Keeping the actual forwarded set is important: it preserves not only the core
but also the additional original blocks retained by the request-length guard.
Protected block membership is therefore independent of whether a protected
embedding introduces a new basis direction.

\FloatBarrier
\section{Algorithm and Implementation Details}
\label{app:algorithm}

\subsection{Complete Blocks and Selection Representations}

A removable interaction block contains an assistant message with one or more
tool calls and all corresponding tool-result messages, matched by call ID.
The parser checks identifier uniqueness and complete matching between calls and results.
The system/task prefix and non-removable messages are preserved independently.
When parsing or reconstructed-request validation fails, the original request
is returned.

The selection representation concatenates assistant reasoning, assistant
text, tool names and arguments, and tool observations under the markers
\texttt{ASSISTANT\_REASONING}, \texttt{ASSISTANT\_TEXT},
\texttt{ACTION}, and \texttt{OBSERVATION}. This encoder input is limited
to 12,000 characters, retaining its beginning and end for overlength blocks.
Only the selection representation is shortened; forwarded blocks use the
complete, unchanged original message objects.
The frozen encoder is Qwen3-Embedding-8B with 1,024 output dimensions.

This block-level interface differs from compressed retrieved text
\citep{xu2024recomp} and learned compact representations such as gist tokens
and in-context autoencoding \citep{mu2023learning,ge2023context}.
In \gem{}, embeddings guide which original blocks to retain; the acting
model continues to receive ordinary messages.

\subsection{Evidence Protection and Shared Settings}
\label{app:evidence-protection}

The protected set is a deduplicated union
\begin{equation}
 \mathcal P_t=
 \mathcal P_t^{\mathrm{recent}}\cup
 \mathcal P_t^{\mathrm{goal}}\cup
 \mathcal P_t^{\mathrm{state}}\cup
 \mathcal P_t^{\mathrm{error}}\cup
 \mathcal P_t^{\mathrm{source}}.
 \label{eq:app-protected-union}
\end{equation}
Recent protection retains four complete blocks of immediate working context.
Goal protection retains the eligible block most similar to the current task.
State protection retains up to three relevant records using the state query,
artifact paths, explicit editing calls, and supported Bash/Python write patterns.
Error protection retains up to two recent error records within eight blocks.
Source protection is activated when source-like blocks constitute at least
half of the history and keeps recent evidence from distinct sources.

These rules distinguish the execution roles of records rather than treating
semantic similarity as interchangeability. State-related and persistent
memory provide the relevant context for this distinction
\citep{chen2026beyond,zhou2026mem1,wu2026remember,hu2026sam}.
The numerical protection limits in Table~\ref{tab:implementation-settings}
are the shared \gem{} settings. A detected write is treated as evidence of
a possible state change; preserving it does not assert that the operation
succeeded.

\begin{table}[!htbp]
\centering
\begin{threeparttable}
\caption{Shared implementation settings across Code, Office, Security, and Web.}
\label{tab:implementation-settings}
\small
\setlength{\tabcolsep}{7pt}
\renewcommand{\arraystretch}{1.12}
\begin{tabular}{ll}
\toprule
\textbf{Setting} & \textbf{Value} \\
\midrule
Minimum complete blocks & 16 \\
Coverage target $\tau$ / nominal core capacity $K$ & 0.90 / 16 \\
Recent / Goal / State / Error limits & 4 / 1 / 3 / 2 \\
Error eligibility horizon & 8 blocks \\
\addlinespace[3pt]
Source protection activation & Source-like fraction $\ge0.50$ \\
Initial-block protection & 0 \\
Candidate / fill policy & \shortstack[l]{\texttt{semantic\_evidence} /\\\texttt{max\_residual}} \\
Maximum request-character reduction $\delta$ & 0.05 \\
Global reselection interval & 256 new blocks \\
\addlinespace[3pt]
Selection-representation length & 12,000 characters \\
Embedding dimension / batch size & 1,024 / 16 \\
Recorded \texttt{budget\_k} & 8 \\
Random-completion seed & 42 \\
Model temperature / context window & 0.0 / 1,000,000 tokens \\
\bottomrule
\end{tabular}
\begin{tablenotes}[flushleft]
\footnotesize
\item[] The core capacity expands to include all required evidence. The recorded
\texttt{budget\_k} is not the final forwarded-block count: the protected set,
core capacity, and request-length guard jointly determine the forwarded set.
\end{tablenotes}
\end{threeparttable}
\end{table}

\subsection{Global Selection and Incremental Forwarding}

Every request is inspected. Requests with fewer than 16 complete blocks remain
unchanged. Global selection is triggered at the first eligible request, after
a session reset/rewrite or goal change, or after 256 newly accumulated blocks.
It initializes the core with protected evidence and adds largest-residual
candidates until coverage, capacity, or candidate exhaustion terminates selection.

The final forwarded set starts from the full safe request. Only blocks outside
the selected core are considered for deletion, with larger redundant blocks
considered first. Each accepted deletion must preserve at least 95\% of the
original request's serialized character count. Ordinary append events keep
the saved actual forwarded set, append new blocks and activated evidence,
and preserve chronological order. Algorithm~\ref{alg:compression} combines
these steps.

\begin{algorithm}[H]
\caption{\gem{} at an intercepted agent request}
\label{alg:compression}
\let\AND\relax
\begin{algorithmic}[1]
\STATE \textbf{Input:} request, goal, saved forwarded state, $\tau$, $K$, $\delta$
\STATE Parse complete blocks and preserve the non-removable prefix
\IF{the tool protocol is invalid or fewer than 16 complete blocks are available}
  \STATE \textbf{return} the original request
\ENDIF
\STATE Construct and deduplicate the protected set $\mathcal P_t$
\IF{global reselection is triggered}
  \STATE Encode blocks; set $B\gets\mathcal P_t$ and $K_t\gets\max\{K,|\mathcal P_t|\}$
  \STATE Construct an orthonormal basis $Q_B$ of the selected representations
  \WHILE{$|B|<K_t$, $E(B)<\tau$, and an independent eligible candidate exists}
    \STATE Select $j\gets\arg\max_{i\notin B}(1-\|Q_B^\top x_i\|_2^2)$ over eligible candidates
    \STATE Add $j$ to $B$ and update the basis and residuals
  \ENDWHILE
  \STATE Initialize $S\gets[n]$; save the original request length $L_0$
  \FOR{eligible blocks $i\notin B$ in the deletion-priority order}
    \IF{$L(S\setminus\{i\})\ge(1-\delta)L_0$}
      \STATE $S\gets S\setminus\{i\}$
    \ENDIF
  \ENDFOR
\ELSE
  \STATE Keep the saved forwarded set and append new blocks and activated evidence
  \STATE Let $S$ be the resulting index set in the current history
\ENDIF
\STATE Reconstruct complete original messages in chronological order
\IF{the reconstructed request fails structural validation}
  \STATE \textbf{return} the original request
\ENDIF
\STATE Save the actual forwarded set and session/goal bookkeeping
\STATE \textbf{return} the reconstructed request
\end{algorithmic}
\end{algorithm}

With $n$ encoded blocks, dimension $d$, and final basis rank $b$, incremental
projection updates in Eq.~\ref{eq:app-residual-update} can be implemented in
$O(ndb)$ arithmetic and $O(nd+db+n)$ working storage, excluding the original
messages and encoder. These are operation-count bounds for incremental updates;
encoding and request serialization are accounted for separately.

\FloatBarrier
\section{Additional Geometry and Evidence Analysis}
\label{app:geometry}

\subsection{Representations, Centering, and Matched Controls}

The isolated-block analysis uses 171 trajectories and 4,701 complete blocks,
encoded independently with frozen Qwen3.5-9B mean layer-32 representations
of dimension 4,096. The selector-space analysis uses 32 trajectories and
1,162 block representations of dimension 1,024. These are the two
representation populations used in Section~\ref{sec:lowdim-result}.

For $X_c=X-\mathbf1\bar x^\top=U\Sigma V^\top$ with positive centered
energy, define
\begin{equation}
 p_j=\frac{\sigma_j^2}{\sum_k\sigma_k^2},\qquad
 r_{\mathrm{eff}}=\exp\!\left(-\sum_jp_j\log p_j\right),\qquad
 \bar r_{90}=\frac1n\min\!\left\{k:\sum_{j\le k}p_j\ge0.90\right\}.
 \label{eq:app-spectrum}
\end{equation}
The centered analysis isolates variation around the mean direction, complementing
studies of representation anisotropy and dominant embedding directions
\citep{ethayarajh2019contextual,mu2017all,gao2021simcse}.

The feature-shift control independently circularly shifts each feature along
the block axis, preserving its marginal values while disrupting cross-feature
alignment. The Gaussian control matches the trajectory matrix dimensions.
Both comparisons are performed within trajectories before aggregation.
Table~\ref{tab:lowdim-main} gives the aggregate dimensionality statistics
corresponding to the main-text comparison.

\begin{table}[!htbp]
\centering
\begin{threeparttable}
\caption{Structured low-dimensionality of agent histories. Values are trajectory
medians for centered representations; $r_{90}/n$ is the fraction of available
blocks required to explain $90\%$ of centered energy.}
\label{tab:lowdim-main}
\small
\setlength{\tabcolsep}{7pt}
\renewcommand{\arraystretch}{1.12}
\begin{tabular}{llrrr}
\toprule
\textbf{Representation} & \textbf{Metric} & \textbf{Real} & \textbf{Feature-shift} & \textbf{Gaussian} \\
\midrule
Isolated blocks & Effective rank & 8.39 & 16.99 & 19.96 \\
Isolated blocks & $r_{90}/n$ & 0.47 & 0.78 & 0.88 \\
\addlinespace[3pt]
Selector space & Effective rank & 14.04 & 23.54 & 30.07 \\
Selector space & $r_{90}/n$ & 0.48 & 0.78 & 0.85 \\
\bottomrule
\end{tabular}
\end{threeparttable}
\end{table}

Real histories have lower effective rank than their feature-shift controls
in all 171 isolated-block comparisons and all 32 selector-space comparisons.
The consistent pattern across the two spaces supports structured geometric
redundancy as a basis for the completion stage.

\subsection{Matched-Budget Evidence Retention}
\label{app:same-budget}

The matched-budget analysis holds the retained block count fixed at each
checkpoint. Both selectors are measured on the same available history and
recorded next action. The online compressor has access only to the current
goal and available history; the recorded next action is used for offline
evaluation.

For action representation $a$, the quantities in
Eq.~\ref{eq:task-support} distinguish projected action-direction energy from
retention of actual neighboring history records. The action projection is
$\|Q_S^\top a\|_2^2$, and Top-3 retention is the fraction of the three
history blocks most similar to $a$ that remain in $S$.
The 650-checkpoint comparison in Table~\ref{tab:matched-budget-main}
shows higher action-related retention with evidence constraints at the same
retained count, while the reported centroid similarity remains 0.98.

Let $f_{jc}$ denote a checkpoint metric for checkpoint $c$ of trajectory $j$,
and let $m_j$ be that trajectory's checkpoint count. The pooled statistic is
\begin{equation}
 \bar f_{\mathrm{pooled}}
 =\frac{\sum_{j=1}^{32}\sum_{c=1}^{m_j}f_{jc}}{\sum_{j=1}^{32}m_j},
 \qquad \sum_{j=1}^{32}m_j=650.
 \label{eq:app-pooled}
\end{equation}
This specifies the weighting used for pooled checkpoint means; trajectories
remain the independent experimental units. No individual task outcomes are
needed to interpret this representation-level comparison.

\subsection{Low-Energy Subspace and Controlled Replacement}
\label{app:directional-intervention}

Let $V_{\mathrm{low}}$ contain the centered-history singular vectors after
the 90\% energy cutoff, and write $\tilde x_i=x_i-\bar x$.
The block's low-energy fraction and its alignment with the action are
\begin{align}
 L_i&=\frac{\|V_{\mathrm{low}}^\top\tilde x_i\|_2^2}
            {\|V^\top\tilde x_i\|_2^2},\label{eq:app-low-energy-fraction}\\
 A_i&=\cos\!\left(V_{\mathrm{low}}V_{\mathrm{low}}^\top\tilde x_i,
                 V_{\mathrm{low}}V_{\mathrm{low}}^\top a\right).
 \label{eq:app-low-energy-alignment}
\end{align}
These quantities are evaluated when their denominators and cosine norms are
nonzero. The targeted condition removes a retained block with high $A_i$;
the matched condition removes a low-alignment block while matching
representation similarity, recency, and serialized length. The two conditions
retain the same number of blocks. These definitions specify the intervention
reported in Table~\ref{tab:targeted-replacement}.

The spectral visualization groups ordered directions into rank bands, whereas
the intervention uses the cumulative-energy cutoff above. Keeping these
definitions separate connects the visualization to the formal replacement
analysis without treating rank fractions as energy fractions.

\FloatBarrier
\section{Ablation and Hyperparameter Analysis}
\label{app:component}

\subsection{Selectors at the Same Retained Count}

Table~\ref{tab:selector-aggregate} compares selection strategies at the same
retained count on 32 trajectories and 59 checkpoints. This comparison isolates
which records are selected from how many records are retained. The evidence-first
selector achieves the highest reported final projection and next-action Top-3
retention while retaining high geometric coverage.

\begin{table}[!htbp]
\centering
\begin{threeparttable}
\caption{Same-retained-count offline selector comparison on 32 trajectories and
59 checkpoints.}
\label{tab:selector-aggregate}
\small
\setlength{\tabcolsep}{8pt}
\renewcommand{\arraystretch}{1.12}
\begin{tabular}{lrrrrr}
\toprule
\textbf{Selector} & \textbf{Action} & \textbf{Final} & \textbf{Goal} & \textbf{Top-3} & \textbf{Centroid} \\
\midrule
Recent-$K$ & 0.96 & 0.88 & 0.82 & 0.76 & 0.97 \\
Random-$K$ & 0.92 & 0.92 & 0.95 & 0.45 & 0.99 \\
Cosine Top-$K$ & 0.86 & 0.91 & 0.97 & 0.53 & 0.96 \\
Novelty Top-$K$ & 0.93 & 0.93 & 0.96 & 0.33 & 0.99 \\
\rowcolor{gemrow}
\gem{} & 0.95 & 0.95 & 0.96 & 0.78 & 0.99 \\
\bottomrule
\end{tabular}
\end{threeparttable}
\end{table}

\subsection{Contribution of Evidence Protection}

Table~\ref{tab:component-full} reports the component ablation at each variant's
natural retained fraction. Relative to geometry-only selection, the complete
protection bundle retains more action-related information. Removing State or
Error protection reduces Top-3 retention, supporting their role within the
combined evidence set. The natural-budget comparison complements the
fixed-count comparison above; their retained fractions and checkpoint protocols
are reported separately.

\begin{table}[!htbp]
\centering
\begin{threeparttable}
\caption{Offline component ablation at each variant's natural retained fraction
(32 trajectories, 59 checkpoints).}
\label{tab:component-full}
\small
\setlength{\tabcolsep}{8pt}
\renewcommand{\arraystretch}{1.12}
\begin{tabular}{lrrrr}
\toprule
\textbf{Variant} & \textbf{Retention} & \textbf{Action} & \textbf{Top-3} & \textbf{Centroid} \\
\midrule
Full constraints & 0.46 & 0.95 & 0.62 & 0.99 \\
Geometry only & 0.39 & 0.91 & 0.25 & 0.98 \\
Recent + Goal only & 0.45 & 0.95 & 0.57 & 0.98 \\
Without State & 0.45 & 0.95 & 0.53 & 0.98 \\
Without Error & 0.45 & 0.95 & 0.56 & 0.98 \\
All candidates & 0.44 & 0.95 & 0.59 & 0.99 \\
Maximum basis $K=12$ & 0.44 & 0.95 & 0.61 & 0.98 \\
Maximum basis $K=20$ & 0.47 & 0.95 & 0.63 & 0.99 \\
\bottomrule
\end{tabular}
\begin{tablenotes}[flushleft]
\footnotesize
\item[] Because retention differs across variants, this table does not isolate
selection quality at identical budgets.  The same-$K$ comparison above provides
the corresponding budget-controlled view.
\end{tablenotes}
\end{threeparttable}
\end{table}

For the online ablation in Table~\ref{tab:online-ablation},
\textbf{Random completion} retains the protected set and candidate pool and
fills remaining slots by seeded sampling to match the greedy reference count.
\textbf{Geometry only} removes the semantic protection groups while retaining
the candidate policy, residual rule, and request-validity bookkeeping.
These definitions separate the roles of evidence protection and geometric
completion without duplicating the online outcome table.

\subsection{Budget and Coverage Sensitivity}
\label{app:sweeps}

The exact-$K$ sweep uses 28 trajectories and 525 checkpoints common to all
budgets. Table~\ref{tab:k-sweep} gives the complete tested range $K=8,\ldots,20$,
including all four selectors. Evidence-constrained selection exceeds
geometry-only Top-3 retention throughout this range.

\begin{table}[h]
\centering
\begin{threeparttable}
\caption{Exact retained-count sweep.  Values are next-action Top-3 retention on
28 trajectories and 525 checkpoints common to all $K$.}
\label{tab:k-sweep}
\small
\setlength{\tabcolsep}{10pt}
\renewcommand{\arraystretch}{1.12}
\begin{tabular}{crrrr}
\toprule
$K$ & \shortstack[r]{\textbf{Geometry}\\\textbf{only}} & \shortstack[r]{\textbf{Evidence}\\\textbf{constrained}} & \textbf{Recent-$K$} & \textbf{Recent+Goal} \\
\midrule
8  & 0.18 & 0.56 & 0.50 & 0.51 \\
9  & 0.21 & 0.56 & 0.53 & 0.52 \\
10 & 0.24 & 0.57 & 0.55 & 0.53 \\
11 & 0.27 & 0.59 & 0.54 & 0.54 \\
12 & 0.31 & 0.60 & 0.56 & 0.56 \\
13 & 0.33 & 0.62 & 0.58 & 0.57 \\
14 & 0.36 & 0.62 & 0.60 & 0.59 \\
15 & 0.42 & 0.64 & 0.63 & 0.61 \\
16 & 0.46 & 0.67 & 0.66 & 0.63 \\
17 & 0.50 & 0.68 & 0.68 & 0.66 \\
18 & 0.54 & 0.71 & 0.69 & 0.67 \\
19 & 0.59 & 0.74 & 0.71 & 0.72 \\
20 & 0.66 & 0.78 & 0.74 & 0.74 \\
\bottomrule
\end{tabular}
\end{threeparttable}
\end{table}

The coverage sweep uses 32 trajectories and 59 checkpoints.
Table~\ref{tab:tau-sweep} shows the retained fraction, next-action Top-3
retention, and action projection for every tested threshold. Increasing
$\tau$ progressively adds geometric coverage and retains additional task evidence.
The $\tau=0.90$ row is the operating point highlighted in the main text.

\begin{table}[!htbp]
\centering
\begin{threeparttable}
\caption{Coverage-threshold sweep on 32 trajectories and 59 checkpoints.}
\label{tab:tau-sweep}
\small
\setlength{\tabcolsep}{9pt}
\renewcommand{\arraystretch}{1.12}
\begin{tabular}{crrr}
\toprule
$\tau$ & \textbf{Retained fraction} & \textbf{Top-3} & \textbf{Action projection} \\
\midrule
0.85 & 0.38 & 0.58 & 0.94 \\
0.86 & 0.39 & 0.58 & 0.94 \\
0.87 & 0.41 & 0.58 & 0.95 \\
0.88 & 0.43 & 0.61 & 0.95 \\
0.89 & 0.44 & 0.62 & 0.95 \\
\rowcolor{gemrow}
\textbf{0.90} & 0.48 & 0.64 & 0.96 \\
0.91 & 0.48 & 0.64 & 0.96 \\
0.92 & 0.49 & 0.64 & 0.96 \\
0.93 & 0.50 & 0.64 & 0.96 \\
0.94 & 0.51 & 0.64 & 0.96 \\
0.95 & 0.53 & 0.65 & 0.97 \\
\bottomrule
\end{tabular}
\end{threeparttable}
\end{table}

These sweeps provide the tabulated values corresponding to
Figure~\ref{fig:sensitivity}. They evaluate representation-level evidence
retention under controlled compression settings.

\section{Evaluation Protocol and Efficiency Analysis}
\label{app:protocol}

\subsection{Task Units and Aggregate Metrics}

The online evaluation uses WorkBuddyBench \citep{team2026tencent} with
DeepSeek-V4-Flash \citep{xu2026deepseek} at temperature zero.
A complete task is the unit for reward and token aggregation. For a domain
$d$ with $N_d$ tasks, define
\begin{equation}
 \mathrm{Score}_d=\frac{100}{N_d}\sum_{i\in d}R_i,
 \qquad
 \mathrm{Tokens}_d=\frac1{N_d}\sum_{i\in d}T_i.
 \label{eq:app-task-aggregation}
\end{equation}
Overall Full260 means weight domains by task count. The fixed-40 panel
contains ten tasks per domain, and the 20-task online ablation contains
five per domain. All three use the same domain order: Code, Office,
Security, and Web.

The trajectory-level geometric analyses evaluate a different question from
end-to-end reward. Long-context understanding is likewise studied through
context-use analyses and dedicated benchmarks
\citep{liu2024lost,bai2024longbench,bai2025longbench}; here the primary online
outcome is the reward of a completed agent task.

\subsection{Token Accounting and Full260 Efficiency}
\label{app:token-accounting}

For each counted call $c$ in task $i$, let $I_{ic}$ and $O_{ic}$ denote
model input and output tokens, and let $T_i^{\mathrm{enc}}$ be embedding
usage. Combined tokens are
\begin{equation}
 T_i=\sum_c(I_{ic}+O_{ic})+T_i^{\mathrm{enc}}.
 \label{eq:app-combined-tokens}
\end{equation}
When usage is partitioned, $I_{ic}=I_{ic}^{\mathrm{uncached}}+
I_{ic}^{\mathrm{cached}}$ and $O_{ic}=O_{ic}^{\mathrm{normal}}+
O_{ic}^{\mathrm{reasoning}}$. The component counts are not added to the
input/output totals a second time. Model tokens are the four model components;
combined tokens additionally include the selector's embedding usage.

\begin{table}[!htbp]
\centering
\begin{threeparttable}
\caption{Full260 token composition, mean per task. K and M denote thousands
and millions of tokens. Values use the same reporting precision as the main text.}
\label{tab:full260-token-components}
\small
\setlength{\tabcolsep}{10pt}
\renewcommand{\arraystretch}{1.12}
\begin{tabular}{lrr}
\toprule
\textbf{Token category} & \textbf{Base agent} & \textbf{\gem{}} \\
\midrule
Uncached input & 48.70K & 201.48K \\
Cached input & 2.60M & 1.82M \\
Normal output & 12.40K & 15.68K \\
Reasoning output & 25.38K & 22.81K \\
\midrule
Model total & 2.69M & 2.06M \\
Embedding & 0 & 46.97K \\
\midrule
\rowcolor{gemrow}
\textbf{Combined} & 2.69M & \textbf{2.11M} \\
\bottomrule
\end{tabular}
\end{threeparttable}
\end{table}

Table~\ref{tab:full260-token-components} explains the net reduction in
Table~\ref{tab:main260}: combined usage decreases from 2.69M to 2.11M
per task after accounting for embedding cost. The reduced cached-input volume
more than offsets the additional embedding and uncached-input volume.
All categories are retained so that the total is interpretable.

Token accounting is distinct from the physical memory management studied in
LLM serving \citep{kwon2023efficient}. The combined-token measure here includes
both cached and uncached input and is used consistently for the efficiency
comparison; it is not presented as a direct wall-clock measurement.

\subsection{Reference Panel and Online Ablation Protocol}
\label{app:fixed40-reference}

The current fixed-40 method comparison is reported in
Table~\ref{tab:method-comparison}; the current 20-task online ablation is
reported in Table~\ref{tab:online-ablation}. This appendix specifies their
aggregation and method definitions rather than reproducing task-level records
or expanding the tables into individual task outcomes.

The fixed-40 panel and the Full260 evaluation retain the comparison scope
stated in the main text. The same-ID \gem{} slice and the reference methods
are not treated as a newly executed paired comparison. The online ablation
uses the same protected-set definitions, candidate policy, and valid-request
structure, with only the completion or protection component changed as
described in Appendix~\ref{app:component}.

\FloatBarrier
\end{document}

%% file: commands.tex
\usepackage[left=2.5cm,
right=2.5cm,
top=2.3cm,
bottom=2.3cm,
headheight=20pt,
headsep=10pt,
footskip=25pt,
letterpaper]{geometry}
\usepackage[utf8]{inputenc}
\usepackage[T1]{fontenc}
\usepackage[english]{babel}
\usepackage{amsmath,amsfonts,amssymb,amsthm,thmtools}
\usepackage{graphicx}
\usepackage{hyperref}
\usepackage{fancyhdr}
\usepackage[normalem]{ulem}
\usepackage{graphicx}
\usepackage{stfloats}
\usepackage{wrapfig}
\usepackage{epstopdf}
\usepackage{cleveref}
\usepackage{subfloat}
\usepackage{subcaption}
\usepackage{xspace}
\usepackage{enumitem}
\usepackage{listings}
\usepackage{titlesec}
\usepackage{etoolbox}
\usepackage{setspace}
\usepackage{changepage}
\usepackage{etoolbox}
\usepackage{multirow}
\usepackage{booktabs}
\usepackage{tabularx}
\usepackage{wrapfig}
\usepackage{svg}
\usepackage[percent]{overpic}
\usepackage[round]{natbib}
\usepackage[colorinlistoftodos, shadow,color=blue!30!white
]{todonotes}
\usepackage{xpatch}
\usepackage{siunitx}

\fancypagestyle{first}{\fancyfoot[R]{\small\thepage}}

\setlist[itemize]{leftmargin=1em,itemsep=0ex,topsep=0ex}
\titlespacing*{\paragraph}{0pt}{0ex plus .1ex}{1ex}
\titlespacing*{\section}{0ex}{2.3ex plus .3ex minus .0ex}{.6ex plus .3ex minus .2ex}
\titlespacing*{\subsection}{0ex}{1.5ex plus .3ex minus .5ex}{.4ex plus .2ex minus .1ex}
\titlespacing*{\subsubsection}{0ex}{1.2ex plus .3ex minus .3ex}{.3ex plus .2ex minus .2ex}

\xapptocmd\normalsize{%
\abovedisplayskip=.8em plus .2em minus .2em
\belowdisplayskip=.6em plus .1em minus .1em
\abovedisplayshortskip=.8em plus .2em minus .2em
\belowdisplayshortskip=.6em plus .1em minus .1em
}{}{}

\setcitestyle{numbers}
\renewcommand{\cite}[1]{\citep{#1}}

\definecolor{mydarkblue}{rgb}{0.0,0.15,0.7}
\hypersetup{%
colorlinks=true,
linkcolor=mydarkblue,
citecolor=mydarkblue,
filecolor=mydarkblue,
urlcolor=mydarkblue}

\makeatletter
  
  \def\AND{%
    \end{tabular}\hfil\linebreak[4]\hfil
    \begin{tabular}[t]{@{}c@{}}\bfseries\ignorespaces
  }

  \renewcommand{\maketitle}{%
    \begingroup
      {\centering\LARGE\@title\par}%
      \vskip 1em
      \centering
      \begin{tabular}[t]{@{}c@{}}\strut\@author\strut\end{tabular}%
      \vskip 0.3in minus 0.1in
    \endgroup
  }
\makeatother


%% file: authors.tex
\author{%
  \mbox{Mingxuan Wang\textsuperscript{1}}\quad \mbox{Fei Luo\textsuperscript{1}}\quad \mbox{Bo Wang\textsuperscript{1}}\quad \mbox{Guorun Yao\textsuperscript{1}}\quad \mbox{Yinglong Guo\textsuperscript{1}}\\[2pt]
  \mbox{Chao Ning\textsuperscript{1}}\quad \mbox{Hongyue Chen\textsuperscript{1}}\quad \mbox{Yanbiao Ma\textsuperscript{2,*}}\quad \mbox{Jungong Han\textsuperscript{3,*}}\\[4pt]
  \textsuperscript{1}TierFlow Team\\
  \textsuperscript{2}Gaoling School of Artificial Intelligence, Renmin University of China\\
  \textsuperscript{3}Tsinghua University\\[3pt]
  \textsuperscript{*}Corresponding authors.\quad \href{mailto:ybma1998@ruc.edu.cn}{\texttt{ybma1998@ruc.edu.cn}}%
}
\hypersetup{pdfauthor={Mingxuan Wang, Fei Luo, Bo Wang, Guorun Yao, Yinglong Guo, Chao Ning, Hongyue Chen, Yanbiao Ma, Jungong Han}}